\documentclass[conference]{IEEEtran}
\IEEEoverridecommandlockouts
\usepackage{cite}
\usepackage{url}
\usepackage{amsmath,amssymb,amsfonts,amsthm}
\usepackage{algorithmic}
\usepackage{graphicx}
\usepackage{textcomp}
\usepackage{xcolor}
\usepackage{tikz}
\usetikzlibrary{calc}
\usepackage{booktabs}
\usepackage{algorithm2e}
\usepackage{hyperref}
\hypersetup{colorlinks=true,linkcolor=black,filecolor=black,urlcolor=black,citecolor=black}
\usepackage{multirow}
\def\BibTeX{{\rm B\kern-.05em{\sc i\kern-.025em b}\kern-.08em
    T\kern-.1667em\lower.7ex\hbox{E}\kern-.125emX}}

\newtheorem{theorem}{Theorem}

\newtheorem{definition}{Definition}%

\begin{document}
%Rendezvous to Optimize Vehicle-to-Vehicle Charging
\title{Vehicle-to-Vehicle Charging: Model, Complexity, and Heuristics\\
\thanks{Funded by Fundação para a Ciência e a Tecnologia (FCT), Portugal, through the CMU Portugal Program under grant 10.54499/PRT/BD/152193/2021 with DOI~\url{https://doi.org/10.54499/PRT/BD/152193/2021}; and partially supported by Instituto de Telecomunicações, Portugal, and FCT, under grants 2022.06780.PTDC (DOI~\url{https://doi.org/10.54499/2022.06780.PTDC}) and 10.54499/UIDB/50008/2020 (DOI~\url{https://doi.org/10.54499/UIDB/50008/2020}).}
}

\author{\IEEEauthorblockN{Cl\'audio Gomes}
\IEEEauthorblockA{\textit{Carnegie Mellon University}, USA\\
\textit{FEUP}, \textit{University of Porto}, Portugal\\
\textit{LIACC}, \textit{University of Porto}, Portugal\\
claudiogomes@cmu.edu}
\and
\IEEEauthorblockN{Jo\~ao Paulo Fernandes}
\IEEEauthorblockA{\textit{New York University Abu Dhabi} \\
\textit{United Arab Emirates}\\
jpf9731@nyu.edu
}
\and
\IEEEauthorblockN{Gabriel Falcao}
\IEEEauthorblockA{\textit{Instituto de Telecomunicações} \\
\textit{Dept. of Electrical and Computer Engineering}\\
University of Coimbra, Portugal \\
gabriel.falcao@uc.pt}
\and
\IEEEauthorblockN{Soummya Kar}
\IEEEauthorblockA{\textit{Dept. of Electrical and Computer Engineering} \\
\textit{Carnegie Mellon University, USA}\\
soummyak@andrew.cmu.edu}
\and
\IEEEauthorblockN{Sridhar Tayur}
\IEEEauthorblockA{\textit{Tepper School of Business} \\
\textit{Carnegie Mellon University, USA}\\
stayur@cmu.edu}
}

\maketitle

\begin{abstract}
The rapid adoption of Electric Vehicles (EVs) poses challenges for electricity grids to accommodate or mitigate peak demand. Vehicle-to-Vehicle Charging (V2VC) has been recently adopted by popular EVs, posing new opportunities and challenges to the management and operation of EVs. We present a novel V2VC model that allows decision-makers to take V2VC into account when optimizing their EV operations. We show that optimizing V2VC is NP-Complete and find that even small problem instances are computationally challenging. We propose R-V2VC, a heuristic that takes advantage of the resulting totally unimodular constraint matrix to efficiently solve problems of realistic sizes. Our results demonstrate that R-V2VC presents a linear growth in the solution time as the problem size increases, while achieving solutions of optimal or near-optimal quality. R\nobreakdash-V2VC can be used for real-world operations and to study what-if scenarios when evaluating the costs and  benefits of V2VC.  %Our study allows for a future cost-benefit analysis of V2VC in real-world areas and EV fleets.
\end{abstract}


\section{Introduction}

Electricity grids are undergoing rapid transformation as the adoption of Electric Vehicles (EVs) increases~\cite{IEAglobalevoutlook,RIETMANN2020121038}. This transformation usually involves building new infrastructure, which can be costly and slow. Other complementary approaches include shaping demand through dynamic pricing and using alternative energy distribution methods~\cite{10333947,10333900,10333930}. One of such promising technologies is Vehicle-to-Vehicle Charging (V2VC), which is a feature already present in some EVs, such as Ford F-150 Lightning and Tesla Cybertruck~\cite{ford}.

In regions lacking adequate infrastructure, V2VC technology has the potential to assist EV owners, whether they are individuals or businesses, reach their destinations with less time and energy expenditure. As a concrete example, imagine the case of EV owners who would have to take a time- and energy-consuming detour so that they could have their EV sufficiently charged to reach the next stop. Instead, V2VC could allow these EV owners to be charged somewhere along the path by another EV, avoiding long detours, saving time and energy, while reducing peak demand on the electricity grid~\cite{ford}.

We develop a novel model that incorporates V2VC in vehicle routing problems (VRPs). The inputs are a network map with the positions of EVs and points where EVs can meet to perform V2VC and the Grid-to-Vehicle Charging (G2VC) locations. Our model can be adapted to many VRP variations, such as setting delivery routes and ride-sharing, by adapting/adding constraints. Initially, we prove the NP-Completeness of the problem. Subsequently, we introduce Restricted-V2VC (R-V2VC), a heuristic that can approximately solve the model in polynomial time with respect to the problem size. To the best of our knowledge, despite the fact that previous work addresses V2VC engineering technology with respect to implementation details and control algorithms~\cite{8913525,8710609,8328689,9526757}, there has been no work modeling V2VC as an optimization problem in the context of VRPs for EV fleets in a region. Our contributions are summarized below.

\begin{itemize}
    \item We designed a nonlinear objective integer constrained model for Vehicle-to-Vehicle Charging in VRPs;
    \item We proved the NP-completeness of the V2VC model;
    \item We developed a heuristic (R-V2VC) that uses a bipartite graph for a simpler representation of the V2VC model;
    \item We benchmarked R-V2VC, showing that it solves increasingly large V2VC problems with a sublinear growth in solution time while finding near-optimal solutions.
\end{itemize}

The remainder of the paper is structured as follows. In Section~\ref{sec3}, we describe the V2VC model, the proof of its NP-completeness, and develop the R-V2VC heuristic. In Section~\ref{sec:results}, we present results on the performance and limitations of the original formulation and R-V2VC. In Section~\ref{sec:conclusion}, we conclude and discuss future work.

\section{Methodology}\label{sec3}

%Our algorithmic work is motivated by an emerging application that considers the Vehicle-to-Vehicle Charging (V2VC) technology in current and upcoming Electric Vehicles (EVs). 
We utilize a nonlinear objective integer constrained optimization (NOICO) framework to represent our model: %The presented model is followed by a proof of its NP-Completeness, motivating the usage of heuristics to find feasible solutions that are good enough for decision-makers.
\begin{equation}
    \label{eq:IP}
    \left(\mathrm{IP}\right)_{A,\boldsymbol{b},\boldsymbol{l},\boldsymbol{u},\mathcal{F}}: \begin{cases}
        \min \mathcal{F}\left(\boldsymbol{x}\right)\\
        A\boldsymbol{x}=\boldsymbol{b},~\boldsymbol{l} \leq \boldsymbol{x} \leq \boldsymbol{u},~\boldsymbol{l},\boldsymbol{x},\boldsymbol{u} \in \mathbb{Z}^n\\
        A \in \mathbb{Z}^{m \times n},~\boldsymbol{b} \in \mathbb{Z}^m
    \end{cases},
\end{equation}
\noindent where $\mathcal{F}\left(\boldsymbol{x}\right)$ can be a nonlinear function.

\subsection{V2VC Model Definition}

\subsubsection{Assumptions}

%The development of the model followed a series of assumptions that were introduced to be able to represent the V2VC problem in a practical way within the structure of an $\left(\mathrm{IP}\right)_{A,\boldsymbol{b},\boldsymbol{l},\boldsymbol{u},f}$. 
We assume the following: i)~time evolution is discretized by equal time steps, that ranges from $0$ (start) to $T$ (end of horizon); ii)~we have a road network map $G$ with meeting points, parking stations, other locations, and road connections; iii)~we control a fleet of EVs capable of V2VC and G2VC; iv)~each EV $i$ is located at an initial location $s_i$ at $t=0$; v)~each EV needs to move to their own destination and should arrive there before time $T$; vi)~the energy transfer of V2VC is discrete, and only a fixed amount of energy can be transferred at each time step; vii)~EVs can only perform V2VC at meeting points; viii)~EVs can only perform G2VC at parking stations; and ix) all routing and charging decisions are made at the beginning of the horizon.

\subsubsection{Parameters}

The model is parameterized as in Table~\ref{tab:parameters}.

\begin{table}
    \centering
    \caption{Parameters in the V2VC model and their description.}
    \begin{tabular}{c|p{6.5cm}}
        \toprule
        Parameter & Description \\
        \midrule
        $V$ & Set of EVs.\\
        $P$ & Set of parking stations.\\
        $M$ & Set of meeting points.\\
        $G$ & Graph with a set of nodes that represent points in the map and a set of arcs that represent road connections. The arcs can be directed or undirected edges. The set of nodes includes $P$, $M$, and other locations (note: $P \cap M = \emptyset$).\\
        $e_p$ & Energy that the parking station $p$ provides to an EV in a single time step during G2VC.\\
        $e_i$ & Energy that the EV $i$ provides to another EV in a single time step during V2VC.\\
        $e_a$ & Energy that traversing arc $a$ of $G$ takes.\\
        $d_a$ & Number of time steps that it takes to traverse the arc $a$.\\
        $\mathrm{SOC}_i$ & Starting energy level of the battery of the EV $i$.\\
        $\mathrm{MAXSOC}_i$ & Maximum energy level of the EV $i$.\\
        $s_i$ & Starting position of the EV $i$, which is a node of $G$.\\
        $f_i$ & Destination of the EV $i$, which is a node of $G$.\\
        $T$ & Number of time steps considered. The time windows of the problem goes from $t=0$ to $t=T-1$.\\
        \bottomrule
    \end{tabular}
    \label{tab:parameters}
\end{table}

\subsubsection{Time-Space Network Graph}

We define the set of nodes $N$ as tuples $(v,t)$, where $v$ is a node from $G$ and $t$ is a time step. We also define $A$ as the set of arcs that connect nodes in $N$. The model uses a modified graph, $G_{\mathrm{TS}} = (N, A)$, called \textit{time-space network graph} (in literature, it is also called \textit{time-expanded network graph}~\cite{timeexpandednetwork}). This graph combines $G$ and $\boldsymbol{d}$, from Table~\ref{tab:parameters}, with discrete time steps going from $t=0$ to $t=T-1$. For EV routing purposes in $G_{\mathrm{TS}}$, each EV cannot choose arcs that share a time step. Moreover, each EV needs to select a set of connected arcs such that they have a valid path from $t=0$ to $t=T-1$. In the end, the selected arcs show the path for all EVs through space and time. For each time step, an EV that is at a node may: \textbf{stay at the current node}---it will be at the same node in the next time step---$(v,t)\rightarrow(v,t+1)$; or \textbf{traverse an arc $\boldsymbol{a}$ to a different node}---it will be at a different node after $d_a$ time steps---$(v,t)\rightarrow(w,t+d_a)$. Hence:

\noindent $\bullet$ $e^\mathrm{TS}_a$---energy that traversing arc $a=(v,t)\rightarrow(w,t+d_a)$ of $G_{\mathrm{TS}}$ takes, such that it is equal to $e_{(v,t)} \times d_a$.

\noindent Note that $G_{\mathrm{TS}}$ is a directional graph going from $t=0$ to $t=T-1$, which implies that an EV cannot go back in time.

\subsubsection{Binary Decision Variables}

The model involves three different types of binary decision variables, $X$ for EV routing, $Y$ for the G2VC schedule, and $Z$ for the V2VC schedule. These variables are described in Table~\ref{tab:v2vvariables}.

\begin{table}
    \centering
    \caption{Variables in the V2VC Model, their description, and their \#.}
    \begin{tabular}{c|p{4.7cm}|c}
        \toprule
        Variables & Description & No. of Variables \\
        \midrule
        $X_{i,a}$ & Assigns EV $i$ to traverse arc~$a$ of $G_{\mathrm{TS}}$. & $\lvert V \rvert \cdot \lvert A \rvert$\\
        \midrule
        $Y_{i,a}$ & \multirow{2}{4.7cm}{For $a$ of the form $(p,t)\rightarrow(p,t+1)$, where $p$ is a parking station, assigns EV $i$ to be charged from the grid at a parking station $p$ between $t$ and $t+1$.} & $\lvert V \rvert \cdot \lvert P \rvert$\\
        & & $\cdot~(T-1)$\\
        & &\\
        & &\\
        \midrule
        $Z_{i,j,a}$ & \multirow{2}{4.7cm}{For $a$ of the form $(m,t)\rightarrow(m,t+1)$, where $m$ is a meeting point, assigns EV $i$ to be charged from EV $j$ on the meeting point $m$ between $t$ and $t+1$.} & $\lvert V \rvert \cdot (\lvert V \rvert - 1)$\\
        & & $\cdot~\lvert M \rvert \cdot (T-1)$\\
        & &\\
        & &\\
        \bottomrule
    \end{tabular}
    \label{tab:v2vvariables}
\end{table}

\subsubsection{Linear Constraints}\label{subsec:v2vconstraints}

\paragraph{Path Constraints} There are $\lvert V \rvert \cdot \lvert N \rvert$ constraints related to routing. The path constraints ensure each vehicle $i$ is assigned a feasible route from $s_i$ to $f_i$.
\begin{align}
    % \begin{aligned}
    %     \sum_{a \in \mathrm{in}(n)} X_{i,a} - \sum_{a \in \mathrm{out}(n)} X_{i,a} = 0,\\
    % \forall n \in N \setminus \left\{(s_i, 0), (f_i,T-1)\right\},~\forall i \in V;
    % \end{aligned}\\
    % \begin{aligned}
    %     - \sum_{a \in \mathrm{out}((s_i, 0))} X_{i,a} = -1,\\ \forall i \in V;
    % \end{aligned}\\
    % \begin{aligned}
    %     \sum_{a \in \mathrm{in}((f_i, T-1))} X_{i,a} = 1,\\ \forall i \in V;
    % \end{aligned}
    \begin{aligned}
        \sum_{a \in \mathrm{in}(n)} X_{i,a} - \sum_{a \in \mathrm{out}(n)} X_{i,a} = \begin{cases}
            0 &\mathrm{if~}  n \neq (s_i, 0)\\& \mathrm{or~} n \neq (f_i,T-1);\\
            -1 &\mathrm{if~}  n = (s_i, 0);\\
            1 & \mathrm{if~}  n = (f_i,T-1);
        \end{cases},\\
    n \in N,~\forall i \in V.
    \end{aligned}
\end{align}

\noindent where $\mathrm{in}(n)$ is the set of arcs entering $n$, and $\mathrm{out}(n)$ is the set of arcs leaving $n$.

\paragraph{Battery Constraints} There are $\lvert V \rvert \cdot (T-1)$ constraints that ensure the batteries of the EVs are between 0 and their maximum charge level. Each constraint includes a slack variable $z_{t,i}$.
\begin{align}
    \begin{aligned}
        \sum_{a\in \mathrm{until}(t)} \left(\vphantom{\sum^V_{j\neq i}} \left(Y_{i,a} - X_{i,a} \right) \times e^\mathrm{TS}_{a}+\sum^V_{j\neq i} \left( Z_{i,j,a} - Z_{j,i,a} \right) \times e_{j}\right)\\+ z_{t,i} = - \mathrm{SOC}_i,~\forall t \in \{ 1, \hdots, T-1 \},~\forall i \in V;
    \end{aligned}\\
    \begin{aligned}
        -\mathrm{MAXSOC}_i \leq z_{t,i} \leq 0,~\forall t \in \{ 1, \hdots, T-1 \},~\forall i \in V.
    \end{aligned}
\end{align}

\noindent where $\mathrm{until}(t)$ is the set of arcs that exist before or in time step $t$.

\paragraph{G2VC Constraints} There are $\lvert V \rvert \cdot \lvert P \rvert \cdot (T-1)$ constraints that ensure that EVs can only perform G2VC at parking stations when they are there. These constraints involve variables $X_{i,a}$ and $Y_{i,a}$. Each constraint includes a slack variable $z_{i,a}$.
\begin{align}
    \begin{aligned}
        Y_{i,((p,t),(p,t+1))} - X_{i,((p,t),(p,t+1))} + z_{i,a} = 0,\\\forall i \in V,~\forall p \in P,~\forall t \in \{ 0, \hdots, T-2 \};
    \end{aligned}\\
    \begin{aligned}
        0 \leq z_{i,a} \leq 1,~\forall i \in V,~\forall p \in P,~\forall t \in \{ 0, \hdots, T-2 \}.
    \end{aligned}
\end{align}

\paragraph{V2VC Constraints} There are $2 \cdot \lvert V \rvert \cdot (\lvert V \rvert - 1) \cdot \lvert M \rvert \cdot (T-1)$ constraints that ensure that EVs can only perform V2VC at meeting points when they are physically there. Moreover, there are another $\lvert V \rvert \cdot (\lvert V \rvert - 1) \cdot \lvert M \rvert \cdot (T-1) / 2$ constraints that ensure that charging is unidirectional (EV A can only charge EV B if EV B is not charging EV A). Each of these constraints creates an additional slack variable.
\begin{align}
    \begin{aligned}
        Z_{i,j,((m,t),(m,t+1))} - X_{i,((m,t),(m,t+1))} + z_{i,j,m,t,1} = 0,\\\forall i \in V,~\forall j \in V \setminus \{i\},~\forall m \in M,~\forall t \in \{ 0, \hdots, T-2 \};
    \end{aligned}\\
    \begin{aligned}
        Z_{i,j,((m,t),(m,t+1))} - X_{j,((m,t),(m,t+1))} + z_{i,j,m,t,2} = 0,\\\forall i \in V,~\forall j \in V \setminus \{i\},~\forall m \in M,~\forall t \in \{ 0, \hdots, T-2 \};
    \end{aligned}\\
% \end{align}
% \begin{align}
    \begin{aligned}
        Z_{i,j,((m,t),(m,t+1))} + Z_{j,i,((m,t),(m,t+1))} + z_{i,j,m,t,3} = 1\\\forall (i,j) \in (V \times V),~\forall m \in M,~\forall t \in \{ 0, \hdots, T-2 \};
    \end{aligned}\\
    \begin{aligned}
        0 \leq z_{i,j,m,t,1} \leq 1,\\\forall i \in V,~\forall j \in V \setminus \{i\},~\forall m \in M,~\forall t \in \{ 0, \hdots, T-2 \};
    \end{aligned}\\
    \begin{aligned}
        0 \leq z_{i,j,m,t,2} \leq 1,\\\forall i \in V,~\forall j \in V \setminus \{i\},~\forall m \in M,~\forall t \in \{ 0, \hdots, T-2 \};
    \end{aligned}\\
    \begin{aligned}
        -1 \leq z_{i,j,m,t,3} \leq 0,\\\forall (i,j) \in (V \times V),~\forall m \in M,~\forall t \in \{ 0, \hdots, T-2 \}.
    \end{aligned}
\end{align}

Additionally, there are $2 \cdot \lvert V \rvert \cdot (T-1)$ constraints that prevent an EV from charging more than one EV simultaneously and from being charged by more than one EV simultaneously:
\begin{align}
    \begin{aligned}
        z_{t,i,1} + \sum_{m \in M} \sum_{j \in V \setminus \{i\}} Z_{i,j,((m,t),(m,t+1))} = 0,\\\forall i \in V,~\forall t \in \{ 0, \hdots, T-2 \};
    \end{aligned}\\
    \begin{aligned}
        z_{t,i,2} + \sum_{m \in M} \sum_{j \in V \setminus \{i\}} Z_{j,i,((m,t),(m,t+1))} = 0,\\\forall i \in V,~\forall t \in \{ 0, \hdots, T-2 \};
    \end{aligned}\\
    \begin{aligned}
        -1 \leq z_{t,i,1} \leq 0,~\forall i \in V,~\forall t \in \{ 0, \hdots, T-2 \};
    \end{aligned}\\
    \begin{aligned}
        -1 \leq z_{t,i,2} \leq 0,~\forall i \in V,~\forall t \in \{ 0, \hdots, T-2 \}.
    \end{aligned}
\end{align}

\noindent with each introducing a slack variable.

\subsubsection{Dimensions}
\label{subsubsec:dimA}

In total, the dimensions of the matrix $A$ are $\left[\lvert V \rvert \cdot \lvert N \rvert + \lvert V \rvert \cdot (T-1) \cdot \left( 3 + \lvert P \rvert + \frac{5}{2} \cdot (\lvert V \rvert - 1) \cdot \lvert M \rvert \right)\right] \times \left[\lvert V \rvert \cdot \lvert A \rvert + \lvert V \rvert \cdot (T-1) \cdot \left( 3 + 2 \cdot \lvert P \rvert + \frac{7}{2} \cdot (\lvert V \rvert - 1) \cdot M \right)\right]$.

\subsubsection{Objective function}

The V2VC model accepts an objective function $\mathcal{F}\left(\boldsymbol{x}\right)$, as shown in~\eqref{eq:IP}, that is minimized under the set of constraints listed above. In practice, $\mathcal{F}\left(\boldsymbol{x}\right)$ can include linear terms such as the cost of the energy consumed at each time step, or the time it takes for EVs to reach their destination. It may also include nonlinear terms that, e.g, balance V2VC efforts among the EV fleet for fairness, or regulate demand for a specific meeting point. Ultimately, $\mathcal{F}\left(\boldsymbol{x}\right)$ depends on the objectives of the decision-maker.

\subsection{NP-Completeness Proof}

The general class of integer problems (IPs) as stated in~\eqref{eq:IP} is known to be NP-complete. However, the introduction of constraints of specific types can lead to specific subclasses that are not NP-complete. Therefore, the implications of the specific constraints in the V2VC model on its complexity need to be investigated. For these purposes, we consider \textit{V2VC problem} in which an answer to a particular instance of the V2VC model is \textit{True} if the constraints $Ax=b,~l \leq x \leq u$ shown in equation \eqref{eq:IP} can be satisfied, and \textit{False} otherwise.

A proof for the NP-completeness of this model is achieved by showing that the decision version of the V2VC problem belongs to the NP class and is NP-hard. Due to space constraints, we present only a synthesis of the NP-hardness proof. The reader can consult the complete proof by accessing \href{https://cfpgomes.github.io/V2VCNP.pdf}{https://cfpgomes.github.io/V2VCNP.pdf}.

\subsubsection{NP class}

For any given decision problem, a \textit{verifier} is an algorithm capable of verifying whether a solution $\tau$ given to an instance $p$ of the problem is correct. Formally, a verifier is a function $\mathcal{V}(p,\tau)$, which returns \textit{True} when $\tau$ is correct and \textit{False} otherwise~\cite{arora_computational_2009}. We define a \textit{witness} $x$ of a decision problem instance $p$ as a solution to $p$ such that $\mathcal{V}(p,x)=\mathit{True}$, where $\mathcal{V}$ is a verifier of $p$. That is, a witness is a correct solution to a problem instance, which, once verified, allows us to conclude that the problem instance is \textit{True}. The decision version of the V2VC problem is in NP if: there exists a polynomial-time verifier that checks that a witness attests to a particular instance of the problem being \textit{True}; and the dimension of the witness is polynomial in the size of the problem instance~\cite{arora_computational_2009}.

For the V2VC problem, we can define a polynomial-time verifier as the algorithm that checks whether the equality and inequality constraints $Ax=b,~l \leq x \leq u$ shown in equation~\eqref{eq:IP} are satisfied by a solution $x$. Concretely, this verifier of the V2VC problem can be defined as
\begin{equation}
    \mathcal{V}_{\text{V2VC}}\left((A,b,l,u),x\right) = \begin{cases}
        \mathit{True} & \text{if } Ax=b,\\&l \leq x \leq u \text{ are satisfied,}\\
        \mathit{False} & otherwise.
    \end{cases}
\end{equation}

\noindent where $(A,b,l,u)$ is a V2VC instance and $x$ is a solution to the instance. We note that this formulation uses the same notation as in equation \eqref{eq:IP} for simplicity. The matrix $A$ and the vectors $b,l,$ and $u$ are represented as a series of constraints (shown in Section~\ref{subsec:v2vconstraints}) and the solution $x$ is represented as the values for the binary decision variables $X$, $Y$, and $Z$, as well as the slack variables (shown in Section~\ref{tab:v2vvariables}). $\mathcal{V}_{\text{V2VC}}$ performs matrix multiplication and equality and inequality comparisons to check whether a given solution is correct for a given V2VC instance. The dimensions of the matrix have been shown in Section~\ref{subsubsec:dimA} to be polynomial in terms of the number of vehicles, parking stations, meeting points, and time steps. Furthermore, the length of the vectors compared in the equality and inequality comparisons is the same as the number of columns in A. Since these operations are known to be computable in polynomial time with respect to the dimensions of the matrix and vectors, $\mathcal{V}_{\text{V2VC}}$ is a polynomial-time verifier.

Taking into account the verifier $\mathcal{V}_{\text{V2VC}}$, a witness $x$ for it would be a set of values attributed to the binary decision variables $X$, $Y$, and $Z$, as well as the slack variables, that satisfy the constraints of the V2VC instance (which is the set of matrices and vectors $A$, $b$, $l$, and $u$). In other words, a witness is a solution for a V2VC instance that is able to route all the vehicles to their destinations while also meeting the V2VC instance's path constraints, battery constraints, parking station charging constraints, and V2VC charging constraints.

Any solution of a V2VC instance is a set of values attributed to its binary decision variables $X$, $Y$, and $Z$, as well as the slack variables. By definition, the length of these variables is the same as the number of columns in the matrix $A$. Again, by definition, a witness is a solution of a V2VC instance. Therefore, a witness of a V2VC instance is polynomial in terms of the size of the instance ($A$, $b$, $l$, and $u$). We have shown that there exists a verifier that checks that a witness attests to a particular instance of the problem to be \textit{True} in time polynomial in the size of the problem instance; and that the size of the witnesses is polynomial in the size of the problem instance. Thus, the V2VC problem is proven to be in NP. \hfill $\square$

\subsubsection{Reduction from the 3SAT problem}

A problem is considered NP-hard if we can reduce all the problems in NP to it in polynomial time. To prove that the V2VC problem is NP-hard, it suffices to show a polynomial-time reduction to it from a known NP-hard problem. More concretely, we will show that the V2VC problem is NP-hard by performing a polynomial-time Karp reduction from the 3SAT problem~\cite{arora_computational_2009,Karp1972}.

\paragraph{3SAT Problem}

The 3SAT problem is a decision problem that is known to be NP-hard and is defined as follows:

\begin{definition}
    The 3SAT problem accepts as input a Boolean formula in \textbf{conjuctive normal form} such that there are at most three literals in each clause. There may be any number of atoms and clauses. The output is \textbf{True} if the formula is \textbf{satisfiable} and \textbf{False} otherwise.
\end{definition}

\noindent We note that a literal is one of two types in the context of the 3SAT: \textit{positive literal} when it is just the atom; and \textit{negative literal} when it is the negation of the atom. Moreover, a Boolean formula in the conjuctive normal form is a conjuction of one or more clauses, where each clause is a disjunction of literals. Finally, a Boolean formula is considered satisfiable when there is an assignment of values to the atoms that will satisfy it. A Boolean formula is considered to be satisfied by an assignment of values to its atoms when it is \textit{True} under that assignment. We introduce the Boolean formula:
\begin{equation}
    \label{eq:3sat}
    \left(x_1\right) \land \left(x_2 \lor x_3\right) \land \left(\lnot x_2 \lor \lnot x_3\right) \land \left(x_1 \lor \lnot x_2 \lor x_3\right).
\end{equation}

Going back to the context of our NP-completeness proof of the V2VC problem, we define a verifier for the 3SAT problem:
\begin{equation}
    \mathcal{V}_{\text{3SAT}}(p,x) = \begin{cases}
        \text{\textit{True} if $p$ is \textit{True} under $x$,}\\
        \text{\textit{False} otherwise,}
    \end{cases}
\end{equation}

\noindent where $p$ is the 3SAT instance, defined by its Boolean formula, and $x$ is a solution to the instance, defined as a list of ($\mathrm{atom}$ = \textit{truth value}) pairs, with one pair for each atom. Since the 3SAT problem is in NP, this verifier can be executed in a time polynomial to the number of atoms and clauses (which is easy to see how, given it only needs to compute a polynomial number of bitwise operations). Now that a polynomial-time verifier $\mathcal{V}_{\text{3SAT}}$ has been defined, we can define a witness $x$ of a 3SAT instance $p$ as a list of (atom, truth value) pairs such that $\mathcal{V}_{\text{3SAT}}(p,x) = \mathit{True}$. In other words, the witness $x$ of a 3SAT instance is an assignment of values to the atoms such that it satisfies the Boolean formula of the instance. As an example, $(x_1 = \mathit{True}, x_2 = \mathit{False}, x_3 = \mathit{True})$ is a witness for the 3SAT instance defined by the Boolean formula in \eqref{eq:3sat}.

\paragraph{Reduction Steps}
\label{subsub:redsteps}

We show next the reduction steps from the 3SAT problem to the V2VC problem. Figure~\ref{fig:visualaid} shows the application of such steps for the 3SAT instance defined by \eqref{eq:3sat}. Let us suppose that we are given an instance of 3SAT with $n$ atoms $x_1, \hdots, x_n$ and $m$ clauses $c_1,\hdots,c_m$. Then we construct the space network graph $G=(N,A)$ and the other parameters for the equivalent V2VC instance as follows:

\noindent $\bullet$ For each clause $c_j$, we create a node $\mathit{sat}_j \in M$, which is a meeting point ($M \subset N$), and a node $f_j \in N$, which are connected by a direct edge from $\mathit{sat}_{j}$ to $f_j$.

\noindent $\bullet$  For each atom $x_i$, we create a starting node $s_{i}\in N$. Next, we create a node $\mathit{true}_i~\in~M$ and a node $\mathit{false}_i~\in~M$, which are meeting points. Afterwards, we create directed edges from $s_{i}$ to $\mathit{true}_i$ and to $\mathit{false}_i$.

\noindent $\bullet$  For each clause $c_j$ and for each atom $x_i$, if $c_j$ has a literal of $x_i$, we create a directed edge from $\mathit{true}_i$ to $f_j$ and from $\mathit{false}_i$ to $f_j$; if $c_j$ has a positive literal of $x_i$, we create a directed edge from $\mathit{true}_i$ to $\mathit{sat}_j$; and if $c_j$ has a negative literal of $x_i$, we create a directed edge from $\mathit{false}_i$ to $\mathit{sat}_j$.

\noindent $\bullet$ For each atom $x_i$, we denote as $k_i$ the number of literals it has among all clauses. Next, we create $k_i$ EVs $v_{i,1}$ to $v_{i,k_i}$, all located on $s_{i}$. For each EV $v_{i,o},~o\in\left\{1,\hdots,k_i \right \}$, we set its final destination as $f_j$, where $c_j$ is the clause containing the $o^\text{th}$ literal of the atom $x_i$. $\text{SOC}_{v_{i,o}}$ is $1$ if $o\neq1$ and $3 k_i + 1$ if otherwise.

\noindent $\bullet$ For each clause $c_j$, we create an EV $v^{\mathit{sat}}_j$ located at $\mathit{sat}_j$ and whose final destination is $f_j$. $\text{SOC}_{v^{\mathit{sat}}_j}$ is $0$.

\noindent $\bullet$ For each EV $v$, we set $e_i=1$ (energy that EV $v$ provides to another EV in a time step) and $\text{MAXSOC}_v$ as a sufficiently large number, such as $3m+1$.

\noindent $\bullet$ All the arcs (or directed edges) created have a time duration of 1 and a traversal cost of 1.

\noindent $\bullet$ Finally, we define $T = 3K + 6$, where $K$ is the greatest number of literals any atom $x_i$ has among all clauses $\left(K=\max\left(k_1,\hdots,k_n\right)\right)$.

\noindent The construction steps listed above define all the parameters required for the input of a V2VC problem. Therefore, the graph $G$ we just constructed will be transformed into a time-space network graph $G_{TS}$, finalizing the setup of a V2VC instance from a 3SAT instance.

\paragraph{Proof of Correctness for the Reduction}

For the complete proof, which includes the proof of correctness for the reduction from the 3SAT to the V2VC problem, due to space constraints, we refer the reader to \href{https://cfpgomes.github.io/V2VCNP.pdf}{\underline{https://cfpgomes.github.io/V2VCNP.pdf}}. To summarize the reduction steps, the satisfiability of each clause is represented by an EV that needs to be charged to reach its destination. If there is an EV capable of charging this EV, then, equivalently, the clause is satisfied by the literal corresponding to the charging EV. If all the EVs reach their destination, the problem is \textit{True}. Figure~\ref{fig:visualaid} illustrates the reduction from the 3SAT instance derived from the Boolean formula in ~\eqref{eq:3sat}.

\subsubsection{Proof of NP-completeness}

It follows from the proof that the V2VC problem is in NP and is NP-hard that, by definition, the V2VC problem is NP-complete.

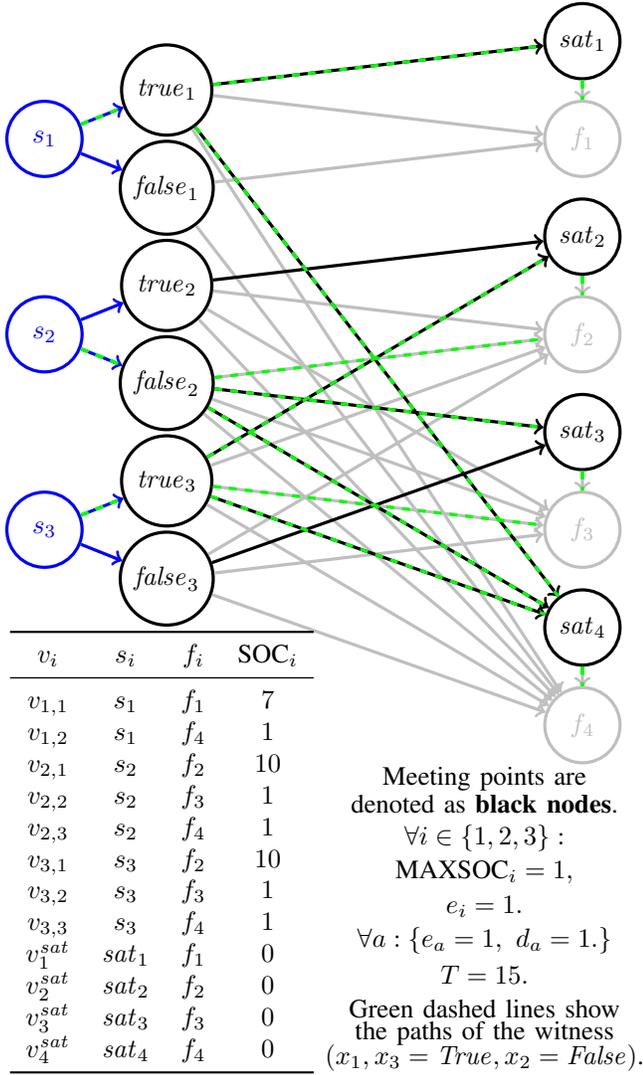
\begin{figure}
    \centering
    \begin{tikzpicture}[scale=0.65]
    % x1
    \node[minimum width=1cm,draw,circle,very thick,blue] (s1) at (-2,-2) {$s_1$};
    \node[minimum width=1.2cm,draw,circle,very thick] (true1) at ($(s1)+(2.5,1)$) {$\mathit{true}_1$};
    \node[minimum width=1.2cm,draw,circle,very thick] (false1) at ($(s1)+(2.5,-1)$) {$\mathit{false}_1$};
    \draw[->,very thick,blue] (s1) -- (true1);
    \draw[->,very thick,blue] (s1) -- (false1);

    % x2
    \node[minimum width=1cm,draw,circle,very thick,blue] (s2) at ($(s1) + (0,-4)$) {$s_2$};
    \node[minimum width=1.2cm,draw,circle,very thick] (true2) at ($(s2)+(2.5,1)$) {$\mathit{true}_2$};
    \node[minimum width=1.2cm,draw,circle,very thick] (false2) at ($(s2)+(2.5,-1)$) {$\mathit{false}_2$};
    \draw[->,very thick,blue] (s2) -- (true2);
    \draw[->,very thick,blue] (s2) -- (false2);

    % x3
    \node[minimum width=1cm,draw,circle,very thick,blue] (s3) at ($(s2) + (0,-4)$) {$s_3$};
    \node[minimum width=1.2cm,draw,circle,very thick] (true3) at ($(s3)+(2.5,1)$) {$\mathit{true}_3$};
    \node[minimum width=1.2cm,draw,circle,very thick] (false3) at ($(s3)+(2.5,-1)$) {$\mathit{false}_3$};
    \draw[->,very thick,blue] (s3) -- (true3);
    \draw[->,very thick,blue] (s3) -- (false3);

    % clause 1: x1
    \node[minimum width=1cm,draw,circle,very thick] (sat1) at (9,0) {$\mathit{sat}_1$};
    \node[minimum width=1cm,draw,circle,very thick,lightgray] (f1) at ($(sat1) + (0,-2)$) {$f_1$};
    
    \draw[->,very thick,lightgray] (sat1) -- (f1);
    
    \draw[->,very thick,lightgray] (true1) -- (f1);
    \draw[->,very thick,lightgray] (false1) -- (f1);

    % clause 2: x2 or x3
    \node[minimum width=1cm,draw,circle,very thick] (sat2) at ($(f1) + (0,-2)$) {$\mathit{sat}_2$};
    \node[minimum width=1cm,draw,circle,very thick,lightgray] (f2) at ($(sat2) + (0,-2)$) {$f_2$};
    
    \draw[->,very thick,lightgray] (sat2) -- (f2);
    
    \draw[->,very thick,lightgray] (true2) -- (f2);
    \draw[->,very thick,lightgray] (false2) -- (f2);
    \draw[->,very thick,lightgray] (true3) -- (f2);
    \draw[->,very thick,lightgray] (false3) -- (f2);

    % clause 3: not x2 or not x3
    \node[minimum width=1cm,draw,circle,very thick] (sat3) at ($(f2) + (0,-2)$) {$\mathit{sat}_3$};
    \node[minimum width=1cm,draw,circle,very thick,lightgray] (f3) at ($(sat3) + (0,-2)$) {$f_3$};

    \draw[->,very thick,lightgray] (sat3) -- (f3);
    
    \draw[->,very thick,lightgray] (true2) -- (f3);
    \draw[->,very thick,lightgray] (false2) -- (f3);
    \draw[->,very thick,lightgray] (true3) -- (f3);
    \draw[->,very thick,lightgray] (false3) -- (f3);

    % clause 4: x1 or not x2 or x3
    \node[minimum width=1cm,draw,circle,very thick] (sat4) at ($(f3) + (0,-2)$) {$\mathit{sat}_4$};
    \node[minimum width=1cm,draw,circle,very thick,lightgray] (f4) at ($(sat4) + (0,-2)$) {$f_4$};
    
    \draw[->,very thick,lightgray] (sat4) -- (f4);
    
    \draw[->,very thick,lightgray] (true1) -- (f4);
    \draw[->,very thick,lightgray] (false1) -- (f4);
    \draw[->,very thick,lightgray] (true2) -- (f4);
    \draw[->,very thick,lightgray] (false2) -- (f4);
    \draw[->,very thick,lightgray] (true3) -- (f4);
    \draw[->,very thick,lightgray] (false3) -- (f4);

    % Satisfiability arcs
    \draw[->,very thick] (true1) -- (sat1);
    \draw[->,very thick] (true2) -- (sat2);
    \draw[->,very thick] (true3) -- (sat2);
    \draw[->,very thick] (false2) -- (sat3);
    \draw[->,very thick] (false3) -- (sat3);
    \draw[->,very thick] (true1) -- (sat4);
    \draw[->,very thick] (false2) -- (sat4);
    \draw[->,very thick] (true3) -- (sat4);

    \node (tab) at ($(s2) + (2.4,-10.6)$) {\begin{tabular}{cccc}
        \toprule
        $v_i$ & $s_i$ & $f_i$ & $\text{SOC}_i$\\
        \midrule
         $v_{1,1}$ & $s_1$ & $f_1$ & $7$\\
         $v_{1,2}$ & $s_1$ & $f_4$ & $1$\\
         $v_{2,1}$ & $s_2$ & $f_2$ & $10$\\
         $v_{2,2}$ & $s_2$ & $f_3$ & $1$\\
         $v_{2,3}$ & $s_2$ & $f_4$ & $1$\\
         $v_{3,1}$ & $s_3$ & $f_2$ & $10$\\
         $v_{3,2}$ & $s_3$ & $f_3$ & $1$\\
         $v_{3,3}$ & $s_3$ & $f_4$ & $1$\\
         $v^{\mathit{sat}}_1$ & $\mathit{sat}_1$ & $f_1$ & $0$\\
         $v^{\mathit{sat}}_2$ & $\mathit{sat}_2$ & $f_2$ & $0$\\
         $v^{\mathit{sat}}_3$ & $\mathit{sat}_3$ & $f_3$ & $0$\\
         $v^{\mathit{sat}}_4$ & $\mathit{sat}_4$ & $f_4$ & $0$\\
         \bottomrule
    \end{tabular}};
    \node (h0) at ($(s2) + (9,-9.1)$) {Meeting points are};
    \node (h1) at ($(h0) + (0,-0.5)$) {denoted as \textbf{black nodes}.};

    \node (h2) at ($(h1) + (0,-0.7)$) {$\forall i \in \left\{1,2,3\right\}:$};
    \node (h3) at ($(h2) + (0,-0.7)$) {$\text{MAXSOC}_i = 1,$};
    \node (h4) at ($(h3) + (0,-0.7)$) {$e_i = 1.$};

    \node (h5) at ($(h4) + (0,-0.7)$) {$\forall a:\left\{e_a = 1,~d_a = 1.\right\}$};
    
    \node (h8) at ($(h5) + (0,-0.7)$) {$T = 15.$};
    
    \node (h9) at ($(h8) + (0,-0.7)$) {Green dashed lines show};
    \node (h10) at ($(h9) + (0,-0.5)$) {the paths of the witness};
    \node (h11) at ($(h10) + (0,-0.5)$) {$(x_1, x_3 = \mathit{True}, x_2 = \mathit{False})$.};

    % witness paths
    \draw[very thick,green,dashed] (s1) -- (true1) -- (sat1) -- (f1);
    \draw[very thick,green,dashed] (s1) -- (true1) -- (sat4) -- (f4);
    
    \draw[very thick,green,dashed] (s2) -- (false2) -- (f2);
    \draw[very thick,green,dashed] (s2) -- (false2) -- (sat3) -- (f3);
    \draw[very thick,green,dashed] (s2) -- (false2) -- (sat4) -- (f4);
    
    \draw[very thick,green,dashed] (s3) -- (true3) -- (sat2) -- (f2);
    \draw[very thick,green,dashed] (s3) -- (true3) -- (f3);
    \draw[very thick,green,dashed] (s3) -- (true3) -- (sat4) -- (f4);

    \draw[very thick,green,dashed] (sat1) -- (f1);
    \draw[very thick,green,dashed] (sat2) -- (f2);
    \draw[very thick,green,dashed] (sat3) -- (f3);
    \draw[very thick,green,dashed] (sat4) -- (f4);
    \end{tikzpicture}
    \caption{Illustration of a reduction from a 3SAT instance to a V2VC instance.}
    \label{fig:visualaid}
    \vspace{-0.64cm}
\end{figure}

\subsection{Restricted-V2VC Heuristic}

Motivated by the NP-Completeness of the V2VC model, we introduce additional restrictions and propose \textit{Restricted-V2VC (R-V2VC)}, a heuristic to approximately solve the V2VC model in polynomial time and space with respect to the problem size.

\subsubsection{Assumptions and Heuristic Description}

A set of assumptions were added to the V2VC model: each EV can have at most one V2VC action at a meeting point during the time horizon: (1) charge another EV or (2) be charged by another EV; and, without loss of generality, as will be seen later, parking stations and G2VC are not considered, as they do not contribute significantly to the complexity of the problem. Given these assumptions, R-V2VC has a simpler representation---a bipartite graph $G_{\mathrm{bipartite}}$, described next:
\vspace{-0.3cm}
\begin{algorithm}
% \caption{$G_{\mathrm{bipartite}}$ generator.}\label{alg:two}
\KwData{Parameters in Table \ref{tab:parameters}}
$N \gets V \cup \{f\}$ \verb\/* f is the destination *\/\;
$E \gets \varnothing$\;
\ForEach{$i \in N$ that can reach their destination}{
  $E \gets E \cup \{a_{i, f}\}$\;
}
\ForEach{$i,j \in V$, $m \in M$, s.t. $\exists a_{i,f}$, $\nexists a_{j,f}$, and both $i$ and $j$ can reach $f$ after $i$ charges $j$ at $m$}{
    $E \gets E \cup \{a_{i,j,m}\}$\;
}
\KwResult{$G_{\mathrm{bipartite}} \gets DirectedGraph(N, E)$}
\end{algorithm}
\vspace{-0.4cm}
\noindent After building $G_{\mathrm{bipartite}}$, we can select one action for each EV by selecting one and only one edge for each EV node. If an edge $a_{i,f}$ is selected, then the EV $i$ goes directly to the destination. If an edge $a_{i,j,m}$ is selected, then the EV $i$ does the action (1) charge the EV $j$ at the meeting point $m$; and the EV $j$ does the action (2) be charged by the EV $i$ at the meeting point $m$. Therefore, we solve R-V2VC by finding the best edge for each EV node in this $G_{\mathrm{bipartite}}$ representation such it minimizes a given objective function $\mathcal{F}^\mathrm{R}$. Note that R\nobreakdash-V2VC is compatible with G2VC: we can add edges between needy EVs and nodes that represent parking stations in $G_{\mathrm{bipartite}}$.\begin{theorem}
    Under the assumption that $\mathcal{F}^\mathrm{R}$ is a separable convex objective function, R-V2VC can be solved in time polynomial to the size of the problem.
\end{theorem}\begin{proof} The problem of finding the best edge for each EV node in $G_{\mathrm{bipartite}} = (V\cup f,E)$ can be represented by finding $\boldsymbol{x}$ such that it minimizes $\mathcal{F}^\mathrm{R}\left( \boldsymbol{x} \right)$, given $A_{\mathrm{bipartite}}\boldsymbol{x}=\boldsymbol{1}$, where $\boldsymbol{1}$ is the $\left| V \right|$-dimensional column vector of ones and $A_{\mathrm{bipartite}}$ is the incidence matrix of $G_{\mathrm{bipartite}}$ minus the row of node $f$. Hence, $\boldsymbol{x}$ is the binary vector of size $\left| E \right|$ such that, for any edge $i$, $x_i=1$  means that the edge $i$ is selected and $x_i=0$ means otherwise. A consequence from the derivation of $A_{\mathrm{bipartite}}$ from a bipartite graph is that it is a totally unimodular matrix. Therefore, for certain classes of $\mathcal{F}^\mathrm{R}$ (including separable convex objective functions), R-V2VC is optimally solved in time polynomial to its size~\cite{4567799,unimodularseparableconvex}.
\end{proof}

\section{Experimental Results and Discussion}\label{sec:results}

In this section, the original formulation is compared with R-V2VC in terms of their solution time, solution quality, and number of variables. Scenarios of increasing size were devised, as listed in Table~\ref{tab:scenarios}. The first set of scenarios, B1--B11, is used to benchmark the formulations according to the number of variables and the solution time. The other set of scenarios, Q1--Q6, is used to compare the quality of the solutions returned by both formulations. All scenarios use an objective function that measures the amount of energy spent within the time windows. We use Gurobi\footnote{Solver specification: \texttt{version 11.0.0 build v11.0.0rc2}. OS and CPU specifications: \texttt{(linux64 - "Ubuntu 22.04.2 LTS") 12th Gen Intel(R) Core(TM) i7-12700H}.} to solve all scenarios with both formulations.

\begin{table}
    \centering
    \caption{Configurations of the different problem scenarios.}
    \begin{tabular}{c|c|c|c|c|c|r|r}
        \toprule
        ID & \multicolumn{2}{|c|}{No. of EVs} & \multicolumn{2}{|c|}{No. of} & $T$ & \multicolumn{2}{|c}{No. of Variables} \\
         & Helper & Needy & \multicolumn{2}{|c|}{Nodes} & & \multicolumn{1}{|c|}{Original} & \multicolumn{1}{|c}{R-V2VC}\\
        \midrule
        B1& 1& 1 & \multicolumn{2}{|c|}{20}& 40 & 13 886& 2\\
        B2& 2& 1 & \multicolumn{2}{|c|}{20}& 40 & 29 019& 3 \\
        B3& 2& 2 & \multicolumn{2}{|c|}{20}& 40 & 49 612& 4 \\
        B4& 4& 2 & \multicolumn{2}{|c|}{20}& 40 & 107 178 & 10 \\
        B5& 6& 3 & \multicolumn{2}{|c|}{20}& 40 & 234 477 & 21 \\
        B6& 8& 4 & \multicolumn{2}{|c|}{20}& 40 & 410 916 & 36 \\
        B7& 10 & 5 & \multicolumn{2}{|c|}{20}& 40 & 636 495 & 48 \\
        B8& 20 & 10& \multicolumn{2}{|c|}{40}& 80 & 10 117 230& 201 \\
        B9& 40 & 20& \multicolumn{2}{|c|}{80}& 160& 161 501 820& 788 \\
        B10 & 60 & 30& \multicolumn{2}{|c|}{120} & 240& 817 616 430& 1 787 \\
        B11 & 80 & 40& \multicolumn{2}{|c|}{160} & 320& 2 583 708 600& 2 234 \\
        Q1 & 1 & 1 & \multicolumn{2}{|c|}{2} & 10 & 252 & 2 \\
        Q2 & 2 & 1 & \multicolumn{2}{|c|}{3} & 10 & 891 & 4 \\
        Q3 & 3 & 2 & \multicolumn{2}{|c|}{5} & 10 & 4 230 & 9 \\
        Q4 & 4 & 2 & \multicolumn{2}{|c|}{6} & 10 & 7 428 & 12\\
        Q5 & 5 & 3 & \multicolumn{2}{|c|}{8} & 10 & 17 064 & 20\\
        Q6 & 6 & 3 & \multicolumn{2}{|c|}{9} & 10 & 24 300 & 24\\
        \bottomrule
    \end{tabular}
    \label{tab:scenarios}
\end{table}

\subsection{Number of variables}

Figure~\ref{fig:numvars} shows the number of variables of both the original formulation and the proposed heuristic with respect to the number of EVs in each scenario. The original formulation explodes in the number of variables as we increase the number of EVs and nodes, as well as the time horizon. Due to this explosion, Gurobi runs into memory limitations when trying to solve scenarios of practical size with the original formulation. In contrast, R-V2VC presents a more proportional and near-linear growth in the number of variables. 

\begin{figure}
    \centering
    \includegraphics[width=\columnwidth]{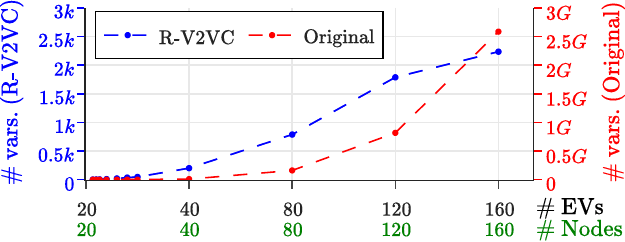}
    \caption{Number of variables used in scenarios B1--B11. Note that the units are different: the blue axis ranges in the thousands ($10^3$), while the red axis ranges in the billions ($10^9$), as shown by the SI prefixes $k$ and $G$, respectively.}
    \label{fig:numvars}
\end{figure}

\subsection{Solution Time}

When trying to solve the scenarios B1--B11 with the original formulation using Gurobi, we are only able to solve for the scenario~B1, with an average solution time of 292ms in 1~000~runs. Gurobi cannot solve the remaining scenarios with the original formulation because it runs out of memory. Figure~\ref{fig:solving_duration} shows the plot of the least squares logarithmic regression of the solution time against the number of variables in the scenario when executing the R-V2VC heuristic with Gurobi. For this plot, we generated 180 random scenarios ranging from 15 to 120~EVs (39 to 3~280~variables). As expected with R\nobreakdash-V2VC, the solution time grows in a sublinear trend compared to the number of variables in the scenario.

\begin{figure}
    \centering
    \includegraphics[width=\columnwidth]{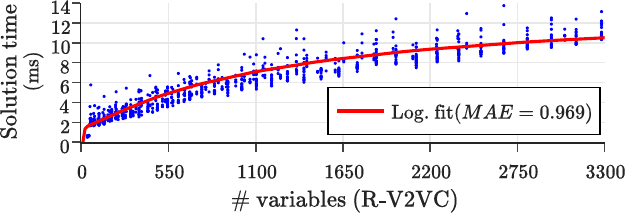}
    \caption{Least Squares Logarithmic Regression of the solution time (ms) against the number of variables in randomly generated scenarios with R-V2VC. The regression is fit with a degree of 4, after which the three first decimal digits of the Mean Squared Error (MSE) do not change. The Mean Absolute Error (MAE) of the fit is $0.969$. A total of 900 data points were computed (36 scenarios times 5 random seeds times 5 runs).}
    \label{fig:solving_duration}
\end{figure}

\subsection{Comparing the Quality of the Solutions}

To gain a more comprehensive understanding of how good the R-V2VC's solutions are, we created another set of small scenarios, Q1--Q6, that allow us to evaluate their proximity to the optimal solutions with respect to the objective function value. Figure~\ref{fig:quality} shows the objective function value of the R-V2VC solutions and compares them with the minimum and maximum objective function values achieved with the original formulation using Gurobi. For each of these scenarios, we observe that R-V2VC gets an optimal or near-optimal solution. Moreover, R-V2VC is orders of magnitude faster in finding solutions compared to the original formulation (e.g. below 2ms versus around 430ms in scenario Q5).

\begin{figure}
    \centering
    \includegraphics[width=\columnwidth]{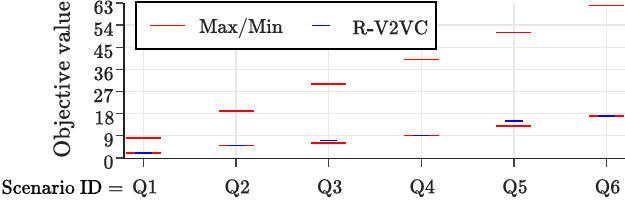}
    \caption{Plot of the objective function value of each of the R-V2VC solutions for the scenarios Q1--Q6 against the minimum and maximum objective function values obtained with the original formulation.}
    \label{fig:quality}
\end{figure}

\subsection{Limitations of R-V2VC and Illustration of a Scenario}

A consequence of the additional assumptions introduced for R-V2VC is that the heuristic cannot solve scenarios where at least one EV needs to perform two or more V2VC actions so that all EVs can reach their destinations during the time horizon. Figure~\ref{fig:unsolvedtimespace} illustrates in space and time a small scenario with three EVs (A, B, and C). For this scenario, Gurobi can find feasible solutions using the original V2VC model. Concretely, the optimal solution found by Gurobi assigns two V2VC actions to EV~B: (1) receive energy from EV~A, and (2) supply energy to EV~C. In contrast, R\nobreakdash-V2VC is infeasible, as there are no feasible solutions that meet R\nobreakdash-V2VC's assumptions. To aid the reader, Figure~\ref{fig:soc} plots the changes in SOC levels (from Table~\ref{tab:parameters}) of the EVs in the same scenario.

\begin{figure}
    \centering
    \includegraphics[width=\columnwidth]{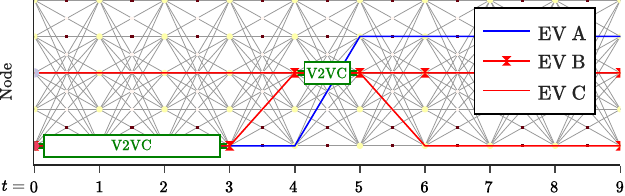}
    \caption{Illustration of a solution for a scenario using the original formulation, in its equivalent $G_{\mathrm{TS}}$. The needy EV B is charged by the helper EV A at $t=0$ for 3 time steps and then charges the other needy EV C at $t=4$.}
    \label{fig:unsolvedtimespace}
\end{figure}

\begin{figure}
    \centering
    \includegraphics[width=\columnwidth]{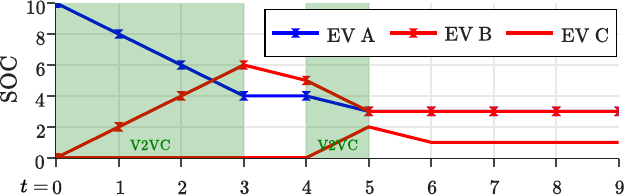}
    \caption{State of Charge (SOC) level changes for each EV in the solution for the small scenario shown in Figure~\ref{fig:unsolvedtimespace}.}
    \label{fig:soc}
\end{figure}

% \subsection{Synthesis of Findings}

% First, our findings validate the NP-completeness of the V2VC model, since the original formulation explodes in the number of variables as we increase the number of EVs and nodes, making Gurobi unable to solve all but one scenario due to running out of memory. Moreover, our findings demonstrate the advantage of R-V2VC's totally unimodular matrix, as the heuristic is capable of solving large-scale scenarios in polynomial time with respect to the scenario size. Next, our findings from scenarios Q1--Q6 also show that the heuristic finds optimal or near-optimal solutions. Last, we demonstrate the limitations of R-V2VC by showing that it cannot solve scenarios where some EV needs to perform more than one action so that all EVs can reach their destination.

\section{Conclusion}\label{sec:conclusion}

In this paper, we proposed a novel Vehicle-to-Vehicle Charging model that allows decision-makers to look at VRPs with V2VC as optimization problems, where the objective function is set according to their specific objectives. This model was shown to be NP-Complete, which motivated our proposed heuristic, Restricted-V2VC (R-V2VC), which efficiently solves the V2VC model for problems of realistic size.

Our findings validate the NP-completeness of the V2VC model, since the original formulation explodes in the number of variables as we increase the number of EVs and nodes. This explosion makes software packages like Gurobi unable to solve any practical problem due to memory constraints. Our findings also confirm a sublinear growth in the solution time of R-V2VC as we increase the problem size. Moreover, we have shown that R-V2VC's finds optimal or near-optimal solutions in scenarios where feasible solutions that meet R-V2VC's assumptions exist, while having a solution time that is orders of magnitude lower. However, R-V2VC's assumptions make it unable to solve scenarios that need one EV to perform more than one V2VC action during the time horizon.

Two significant paths for further work were identified: (1) our V2VC approach could be expanded to accommodate a rolling time horizon, in which EVs can enter and leave at any moment; and (2) R-V2VC's limitation can be mitigated by creating a $k$-indexed hierarchy of heuristics that can have at most $k$ charging actions, with R-V2VC being $k=1$. Regarding the exact solver of the V2VC model, a possible research direction would be to improve it with special methods that avoid out-of-memory issues, such as column-generation~\cite{fischer_dynamic_2014,a14010021}. Ultimately, we believe that our work can be used to research what-if scenarios that assess the effect of V2VC on fleet operations, as well as on the electricity grid, especially in areas that are shifting rapidly to EV mobility.

\bibliographystyle{IEEEtran}
\bibliography{references.bib}

\begin{thebibliography}{10}
\providecommand{\url}[1]{#1}
\csname url@samestyle\endcsname
\providecommand{\newblock}{\relax}
\providecommand{\bibinfo}[2]{#2}
\providecommand{\BIBentrySTDinterwordspacing}{\spaceskip=0pt\relax}
\providecommand{\BIBentryALTinterwordstretchfactor}{4}
\providecommand{\BIBentryALTinterwordspacing}{\spaceskip=\fontdimen2\font plus
\BIBentryALTinterwordstretchfactor\fontdimen3\font minus \fontdimen4\font\relax}
\providecommand{\BIBforeignlanguage}[2]{{%
\expandafter\ifx\csname l@#1\endcsname\relax
\typeout{** WARNING: IEEEtran.bst: No hyphenation pattern has been}%
\typeout{** loaded for the language `#1'. Using the pattern for}%
\typeout{** the default language instead.}%
\else
\language=\csname l@#1\endcsname
\fi
#2}}
\providecommand{\BIBdecl}{\relax}
\BIBdecl

\bibitem{IEAglobalevoutlook}
\BIBentryALTinterwordspacing
IEA, ``\BIBforeignlanguage{en-GB}{Global {EV} {Outlook} 2023},'' 2023. [Online]. Available: \url{https://www.iea.org/reports/global-ev-outlook-2023}
\BIBentrySTDinterwordspacing

\bibitem{RIETMANN2020121038}
N.~Rietmann, B.~Hügler, and T.~Lieven, ``Forecasting the trajectory of electric vehicle sales and the consequences for worldwide co2 emissions,'' \emph{Journal of Cleaner Production}, vol. 261, p. 121038, 2020.

\bibitem{10333947}
E.~Balogun, S.~Martin, A.~Degleris, and R.~Rajagopal, ``Equitable dynamic electricity pricing via implicitly constrained dual and subgradient methods,'' in \emph{2023 IEEE International Conference on Communications, Control, and Computing Technologies for Smart Grids}, 2023, pp. 1--7.

\bibitem{10333900}
G.~Hoogsteen, L.~Winschermann, B.~Nijenhuis, N.~B. Arias, and J.~L. Hurink, ``Robust and predictive charging of large electric vehicle fleets in grid constrained parking lots,'' in \emph{2023 IEEE SmartGridComm}, 2023.

\bibitem{10333930}
I.~Shilov, H.~Le~Cadre, A.~Bušić, A.~Sanjab, and P.~Pinson, ``Towards forecast markets for enhanced peer-to-peer electricity trading,'' in \emph{2023 IEEE International Conference on Communications, Control, and Computing Technologies for Smart Grids (SmartGridComm)}, 2023, pp. 1--7.

\bibitem{ford}
\BIBentryALTinterwordspacing
{Ford Newsroom}, ``\BIBforeignlanguage{en-US}{Lend a hand – and a few miles – to your friends using vehicle-to-vehicle charging on {F-150 Lightning}, {F-150 Hybrid}},'' 2021. [Online]. Available: \url{https://media.ford.com/content/fordmedia/fna/us/en/news/2021/12/21/vehicle-to-vehicle-charging.html}
\BIBentrySTDinterwordspacing

\bibitem{8913525}
G.~Li, Q.~Sun, L.~Boukhatem, J.~Wu, and J.~Yang, ``Intelligent vehicle-to-vehicle charging navigation for mobile electric vehicles via vanet-based communication,'' \emph{IEEE Access}, vol.~7, pp. 170\,888--170\,906, 2019.

\bibitem{8710609}
A.-M. Koufakis, E.~S. Rigas, N.~Bassiliades, and S.~D. Ramchurn, ``Offline and online electric vehicle charging scheduling with {V2V} energy transfer,'' \emph{IEEE Transactions on Intelligent Transportation Systems}, 2020.

\bibitem{8328689}
G.~Li, L.~Boukhatem, L.~Zhao, and J.~Wu, ``Direct vehicle-to-vehicle charging strategy in vehicular ad-hoc networks,'' in \emph{9th IFIP International Conference on New Technologies, Mobility and Security}, 2018.

\bibitem{9526757}
M.~Shurrab, S.~Singh, H.~Otrok, R.~Mizouni, V.~Khadkikar, and H.~Zeineldin, ``An efficient vehicle-to-vehicle ({V2V}) energy sharing framework,'' \emph{IEEE Internet of Things Journal}, 2022.

\bibitem{timeexpandednetwork}
N.~Boland, M.~Hewitt, L.~Marshall, and M.~Savelsbergh, ``The continuous-time service network design problem,'' \emph{Operations Research}, vol.~65, no.~5, pp. 1303--1321, 2017.

\bibitem{arora_computational_2009}
S.~Arora and B.~Barak, \emph{Computational complexity: a modern approach}.\hskip 1em plus 0.5em minus 0.4em\relax Cambridge ; New York: Cambridge University Press, 2009.

\bibitem{Karp1972}
R.~M. Karp, \emph{Reducibility among Combinatorial Problems}.\hskip 1em plus 0.5em minus 0.4em\relax Boston, MA: Springer US, 1972, pp. 85--103.

\bibitem{4567799}
M.~Yannakakis, ``On a class of totally unimodular matrices,'' in \emph{21st Annual Symposium on Foundations of Computer Science}, 1980.

\bibitem{unimodularseparableconvex}
A.~V. Karzanov and S.~T. McCormick, ``Polynomial methods for separable convex optimization in unimodular linear spaces with applications,'' \emph{SIAM Journal on Computing}, vol.~26, no.~4, pp. 1245--1275, 1997.

\bibitem{fischer_dynamic_2014}
F.~Fischer and C.~Helmberg, ``\BIBforeignlanguage{en}{Dynamic graph generation for the shortest path problem in time expanded networks},'' \emph{\BIBforeignlanguage{en}{Mathematical Programming}}, vol. 143, no. 1-2, pp. 257--297, Feb. 2014.

\bibitem{a14010021}
C.~Hansknecht, I.~Joormann, and S.~Stiller, ``Dynamic shortest paths methods for the time-dependent tsp,'' \emph{Algorithms}, vol.~14, no.~1, 2021.

\end{thebibliography}

\end{document}